%% file: main.tex
\documentclass[sigconf, nonacm, table]{acmart}

\pdfoutput=1

\AtBeginDocument{%
  \providecommand\BibTeX{{%
    \normalfont B\kern-0.5em{\scshape i\kern-0.25em b}\kern-0.8em\TeX}}}

\setcopyright{acmcopyright}
\copyrightyear{2020}
\acmYear{2020}
\acmDOI{TBD}

\include{commands}

\begin{document}

% \rowcolors{2}{gray!20}{white}

% \title{SoK: A technical review of Algorithmic Recourse:\\Counterfactual Explanations and Consequential Interventions}
\title{Algorithmic Recourse:\\from Counterfactual Explanations to Interventions}


\author{Amir-Hossein Karimi}
% \email{amirhkarimi@gmail.com}
\affiliation{%
  \institution{MPI-IS, Germany}
  \institution{ETH Z\"urich, Switzerland}
}
\author{Bernhard Sch\"olkopf}
\affiliation{%
  \institution{MPI-IS, Germany}
}
\author{Isabel Valera}
\affiliation{%
  \institution{MPI-IS, Germany}
  \institution{Saarland University, Germany}
}

\renewcommand{\shortauthors}{Karimi, Sch{\"o}lkopf, Valera}

% \begin{abstract}
%   TBD
% \end{abstract}

% %%
% %% The code below is generated by the tool at http://dl.acm.org/ccs.cfm.
% %% Please copy and paste the code instead of the example below.



% %% A "teaser" image appears between the author and affiliation
% %% information and the body of the document, and typically spans the


\begin{abstract}
\input{000abstract.tex}
\end{abstract}

\maketitle

\newcommand{\negspace}{-0.1em}

\vspace{\negspace}
\vspace{5mm}
\section{Introduction}
\label{sec:010introduction}
\vspace{\negspace}
\input{010introduction.tex}

\vspace{\negspace}
\section{Algorithmic Recourse via Counterfactual Explanations}
\label{sec:020background}
\vspace{\negspace}
\input{020counterfactual_explanations.tex}

\vspace{\negspace}
\section{A Causal Perspective of Algorithmic recourse} %: Actions as Interventions}
\label{sec:030causal_interpretation}
\vspace{\negspace}
\input{030causal_interpretation.tex}

\vspace{\negspace}
\section{Algorithmic Recourse via\\Minimal Interventions}
\label{sec:040methodology}
\vspace{\negspace}
\input{040methodology.tex}

\vspace{\negspace}
\section{Towards Realistic Interventions}
\label{sec:050realistic}
\vspace{\negspace}
\input{050realistic.tex}

\vspace{\negspace}
\section{Discussion}
\label{sec:060discussion}
\vspace{\negspace}
\input{060discussion.tex}

% \section*{Broader Impact}
% \input{070broader_impact.tex}

% \clearpage

% \begin{ack}
% The authors would like to thank Adrián Javaloy Bornás and Julius von K\"ugelgen for their valuable feedback on drafts of the manuscript.
% \end{ack}

\begin{acks}
The authors would like to thank Adrián Javaloy Bornás and Julius von K\"ugelgen for their valuable feedback on drafts of the manuscript.
\end{acks}

\clearpage
% \small
\bibliographystyle{ACM-Reference-Format}
\bibliography{references}

\newpage
\appendix
\numberwithin{equation}{section}

\section{Proofs}
\label{app:proofs}
\input{appendix/proofs}

% \section{Working example}
% \label{app:working_example}
% \input{appendix/example}

% \section{Towards Realistic Interventions}
% \label{app:realistic}
% \input{appendix/realistic.tex}

\end{document}

%% file: commands.tex
\usepackage{lipsum}
\usepackage{color, soul}

\usepackage{enumerate}
\usepackage{bbm}
\usepackage{amsmath}
\usepackage{amsmath,amsthm} % amssymb % error: https://pctex.com/kb/86.html
\usepackage{mathtools}
\usepackage{listings}
\usepackage{multirow}
\usepackage{hhline}
\usepackage{changepage} % used for adjustwidth (for large tables)
\usepackage{adjustbox}
\usepackage{subcaption} % subfigure is obsolete and doesn't work with listings; using subcaption instead
\usepackage{url}
\usepackage{wrapfig}
\usepackage{booktabs}
\usepackage{siunitx}
\newcommand{\argmin}{\mathop{argmin}}

\usepackage{tabularx}

\usepackage{enumitem}

\usepackage{fontawesome}

\newcommand{\zvec}[1]{\boldsymbol{#1}} % for normal variables
\newcommand{\zset}[1]{\mathbb{#1}}
\newcommand{\zrv}[1]{\mathrm{#1}} % for random variables (assume no vectors)

\newcommand{\htheta}{h}

\newcommand{\xF}{\zvec{x}^\texttt{F}}
\newcommand{\xFScalar}{x^\texttt{F}}

\newcommand{\xSCFScalar}{x^\texttt{SCF}}

\newcommand{\xCFE}{\zvec{x}^\texttt{CFE}}
\newcommand{\xCFEnearest}{\zvec{x}^\texttt{*CFE}}
\newcommand{\xSCF}{\zvec{x}^\texttt{SCF}}
\newcommand{\xSCFnearest}{\zvec{x}^\texttt{*SCF}}
\newcommand{\xAny}{\zvec{x}}

\newcommand{\doop}{\mathrm{do}}
\newcommand{\pa}{\zvec{\mathrm{pa}}}
\newcommand{\paf}{\zvec{\mathrm{pa}}^\texttt{F}}
\newcommand{\pascf}{\zvec{\mathrm{pa}}^\texttt{SCF}}

\newcommand{\scmModel}{\mathcal{M}}
\newcommand{\scmGraph}{\mathcal{G}}
\newcommand{\ffinverse}[1]{\zset{F}(\zset{F}^{-1}({#1}))}
\newcommand{\fafinverse}[1]{\zset{F}_{\actionSet}(\zset{F}^{-1}({#1}))}
\newcommand{\fastarfinverse}[1]{\zset{F}_{\actionSetStar}(\zset{F}^{-1}({#1}))}

\newcommand{\deltaVector}{\zvec{\delta}}
\newcommand{\deltaVectorStar}{\zvec{\delta}^*}
\newcommand{\deltaScalar}{\delta}
\newcommand{\deltaScalarStar}{\delta^*}

\newcommand{\actionScalar}{a}

\newcommand{\actionSet}{\mathbf{A}}
\newcommand{\actionSetStar}{\mathbf{A}^*}
\newcommand{\actionCFE}{\mathbf{A}^{\texttt{CFE}}}

    \newcommand{\amirm}[1]{}
    \newcommand{\amirr}[1]{}
\usepackage{thmtools, thm-restate}

\theoremstyle{definition}
\newtheorem{definition}{Definition}[section]

\usepackage{tikz}
\usetikzlibrary{shapes,decorations,arrows,calc,arrows.meta,fit,positioning}
\tikzset{
    -Latex, auto, node distance = 0.5 cm and 0.5 cm, semithick,
    state/.style = {circle, draw, minimum width = 0.7 cm},
    inter/.style = {rectangle, draw, minimum width = 0.7 cm, minimum height = 0.7 cm},
    point/.style = {circle, draw, inner sep = 0.04cm, fill, node contents = {}},
    bidirected/.style = {Latex-Latex,dashed},
    el/.style = {inner sep=2pt, align=left, sloped}
}

\usepackage{tikz}
\usetikzlibrary{trees,calc,fadings,decorations.pathreplacing, intersections}

\tikzset{%
  >=latex, % option for nice arrows
  inner sep=0pt,%
  outer sep=2pt,%
  mark coordinate/.style={inner sep=0pt,outer sep=0pt,minimum size=3pt,
    fill=black,circle}%
}

%% file: 000abstract.tex
As machine learning is increasingly used to inform consequential decision-making (e.g., pre-trial bail and loan approval), it becomes important to explain how the system arrived at its decision, and also suggest actions to achieve a favorable decision.
Counterfactual explanations --``how the world would have (had) to be different for a desirable outcome to occur''-- aim to satisfy these criteria.
Existing works have primarily focused on designing algorithms to obtain counterfactual explanations for a wide range of settings.
However, it has largely been overlooked that ultimately, one of the main objectives is to allow people to act rather than just understand.
In layman's terms, counterfactual explanations inform an individual where they need to get to, but not how to get there.
In this work, we rely on causal reasoning to caution against the use of counterfactual explanations as a recommendable set of actions for recourse.
Instead, we propose a shift of paradigm from \emph{recourse via nearest counterfactual explanations} to \emph{recourse through minimal interventions}, shifting the focus from explanations to interventions. 

%% file: 010introduction.tex
\begin{figure}[t]
  \begin{subfigure}[c]{.4\linewidth}
    \begin{center}
      \begin{tikzpicture}
        % x node set with absolute coordinates
        \node[state, fill=gray!60] (x1) at (0,0) {$\zrv{X}_1$};
        \node[state, fill=gray!60] (x2) [below = 0.6 cm of x1] {$\zrv{X}_2$};
        \node[state, fill=gray!60] (y)  [below right = 0.3 cm and 0.8 cm of x1] {$\hat{\zrv{Y}}$};
        
        \node[state] (u1) [left= 0.3cm of x1] {$\zrv{U}_1$};
        \node[state] (u2) [left= 0.3cm of x2] {$\zrv{U}_2$};

        % Directed edge
        \path (u1) edge (x1);
        \path (u2) edge (x2);
        \path (x1) edge (x2);
        \path (x1) edge [bend left  = 10] (y);
        \path (x2) edge [bend right = 10] (y);
      \end{tikzpicture}
    \end{center}
  \end{subfigure}%
  \begin{subfigure}[c]{.6\linewidth}
    \begin{center}
      \[
      \left.
      \begin{aligned}
        \zrv{X}_1 &\coloneqq \zrv{U}_1                     \\
        \zrv{X}_2 &\coloneqq f_2 (\zrv{X}_1) + \zrv{U}_2 ~~\\
      \end{aligned}
      \right\} ~~ \scmModel 
      \]
      \[
      \begin{aligned}
        \hat{\zrv{Y}} &= \htheta(\zrv{X}_1, \zrv{X}_2)
      \end{aligned}
      \]
    %   \left.
    %   \begin{aligned}
    %     \zrv{X}_1 &\coloneqq \zrv{U}_1                     \\
    %     \zrv{X}_2 &\coloneqq f_2 (\zrv{X}_1) + \zrv{U}_2 ~~\\
    %   \end{aligned}
    %   \right\} ~~ \scmModel \\
    %   \begin{aligned}
    %     \hat{\zrv{Y}} &= \htheta(\zrv{X}_1, \zrv{X}_2)
    %   \end{aligned}
    % \[
    %   \scmModel \begin{cases}
    %     \zrv{X}_1 &\coloneqq \zrv{U}_1                     \\
    %     \zrv{X}_2 &\coloneqq f_2 (\zrv{X}_1) + \zrv{U}_2 ~~\\
    %   \end{cases}\\
    % %   \right\} ~~ \scmModel \\
    %   \begin{aligned}
    %     \hat{\zrv{Y}} &= \htheta(\zrv{X}_1, \zrv{X}_2)
    %   \end{aligned}
    %  \]
    \end{center}
  \end{subfigure}
  \caption{Illustration of an example causal generative process governing the world, showing both the graphical model, $\scmGraph$, and the structural causal model, $\scmModel$, \cite{pearl2000causality}.
  In this example, $\zrv{X}_1$ represents an individual's annual salary, $\zrv{X}_2$ is bank balance, and $\hat{\zrv{Y}}$ is the output of a fixed deterministic predictor $\htheta$, predicting the eligibility of an individual to receive a loan.
  }
  \label{figure:working_example}
\end{figure}
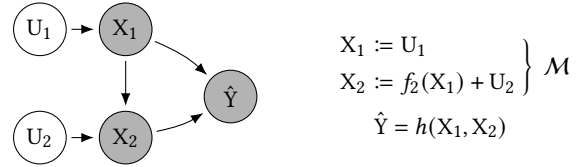

Predictive models are being increasingly used to support consequential decision-making in a number of contexts, e.g., denying a loan, rejecting a job applicant, or prescribing life-altering medication.
As a result, there is mounting social and legal pressure \cite{voigt2017eu} to provide explanations that help the affected individuals to understand ``why a prediction was output'', as well as  ``how to act'' to obtain a desired outcome. 
Answering these questions, for the different stakeholders involved, is one of the main goals of explainable machine learning \cite{doshi2017towards, gunning2019darpa, kodratoff1994comprehensibility, lipton2018mythos, murdoch2019definitions, rudin2019stop, ruping2006learning}.

In this context, several works have proposed to explain a model's predictions of an affected individual  using \emph{counterfactual explanations}, which are defined as statements of ``how the world would have (had) to be different for a desirable outcome to occur'' \cite{wachter2017counterfactual}.
Of specific importance are \emph{nearest counterfactual explanations}, presented as the most similar \emph{instances} to the feature vector describing the individual, that result in the desired prediction from the model \cite{karimi2020model, laugel2017inverse}.
A closely related term is \emph{algorithmic recourse} -- the actions required for, or ``the systematic process of reversing unfavorable decisions by algorithms and bureaucracies across a range of counterfactual scenarios'' -- which is argued as the underwriting factor for temporally extended agency and trust \cite{venkatasubramanianphilosophical}.

Counterfactual explanations have shown promise for practitioners and regulators to validate a model on metrics such as fairness and robustness \cite{karimi2020model, sharma2019certifai, ustun2019actionable}. However, in their raw form, such explanations do not seem to fulfill one of the primary objectives of ``explanations as a means to help a data-subject \emph{act} rather than merely \emph{understand}'' \cite{wachter2017counterfactual}.

The translation of counterfactual explanations to recourse actions, i.e.,  to a recommendable set of actions to help an individual to achieve a favourable outcome,  was first explored in~\cite{ustun2019actionable}, where additional \emph{feasibility} constraints were imposed to support the concept of actionable features (e.g., prevent asking the individual to reduce their age or change their race).  
While a step in the right direction, this work and others that followed \cite{karimi2020model, mothilal2019explaining, poyiadzi2019face, sharma2019certifai} implicitly assume that the set of actions resulting in the desired output would directly follow from the counterfactual explanation.
This arises from the assumption that ``what would \emph{have had to be} in the past'' (retrodiction) not only translates to ``what \emph{should be} in the future'' (prediction) but also to ``what \emph{should be done} in the future'' (recommendation)~\cite{sep-counterfactuals}.
We challenge this assumption and attribute the shortcoming of existing approaches to their lack of consideration for real-world properties, specifically the \emph{causal relationships} governing the world in which actions will be performed.

% \newpage
For ease of exposition, we present the following examples (see \cite{barocas2020hidden} for additional examples).

\textbf{Example 1:} Consider, for example, the setting in Figure \ref{figure:working_example} where an individual has been denied a loan and seeks an explanation and recommendation on how to proceed.
This individual has an annual salary ($\zrv{X}_1$) of $\$75,000$ and an account balance ($\zrv{X}_2$) of $\$25,000$ and the predictor grants a loan based on the binary output of $\htheta = \mathrm{sgn}(\zrv{X}_1 + 5 \cdot \zrv{X}_2 - \$225,000)$.
Existing approaches may identify nearest counterfactual explanations as another individual with an annual salary of $\$100,000$ ($+\%33$) or a bank balance of $\$30,000$ ($+\%20$), therefore encouraging the individual to reapply when either of these conditions are met.
On the other hand, bearing in mind that actions take place in a world where home-seekers save $\%30$ of their salary (i.e., $\zrv{X}_2 \coloneqq 3 / 10 \cdot \zrv{X}_1 + \zrv{U}_2$), a salary increase of only $\%14$ to $\$85,000$ would automatically result in $\$3,000$ additional savings, with a net positive effect on the loan-granting algorithm's decision.

\textbf{Example 2:} Consider now another setting of Figure \ref{figure:working_example} where an agricultural team wishes to increase the yield of their rice paddy.
While many factors influence yield = $\htheta($temperature, solar radiation, water supply, seed quality, ...), the primary actionable capacity of the team is their choice of paddy location.
Importantly, the altitude at which the paddy sits has an effect on other variables.
For example, the laws of physics may imply that a 100$m$ increase in elevation results in a \ang{1}C decrease in temperature on average.
Therefore, it is conceivable that a counterfactual explanation suggesting an increase in elevation for optimal yield, without consideration for downstream effects of the elevation increase on other variables, may actually result in the prediction \emph{not} changing.

The two examples above illustrate the pitfalls of generating recourse actions directly from counterfactual explanations without consideration for the structure of the world in which the actions will be performed.
Actions derived directly from counterfactual explanations may ask too much effort from the individual (\textbf{Example~1}) or may not even result in the desired output (\textbf{Example~2}).

In this paper, we remedy this situation via a fundamental reformulation of the recourse problem, 
where we rely on causal reasoning to incorporate knowledge of causal dependencies into the process of recommending recourse actions, that if acted upon would result in a counterfactual instance that favourably changes the output of the predictive model.
In more detail, we first provide a causal analysis to illuminate the intrinsic limitations of the setting in which actions directly follow counterfactual explanations.  
Importantly, we show that even when equipped with knowledge of causal dependencies after-the-fact, the actions derived from pre-computed (nearest) counterfactual explanations may prove sub-optimal, or directly, unfeasible.
Second, to address the above limitations, we emphasize that, from a causal perspective, actions correspond to interventions which not only model the change in the intervened-upon variable, but also the downstream effects of this intervention on the rest of the (non-intervened-upon) variables.
This insight allows us to propose a \emph{recourse through minimal interventions} problem, whose solution informs stakeholders on how to act in addition to understand. 
We complement this result with a commentary on the form of interventions, and with a more  general definition of feasibility beyond actionability. 
Finally, we provide a detailed discussion on both the importance  and the practical limitations of incorporating causal reasoning in the formulation of  recourse. % the formulation of the algorithmic recourse problem. 

%% file: 020counterfactual_explanations.tex
\emph{Counterfactual explanations} (CFE) are statements of ``how the world would have (had) to be different for a desirable outcome to occur'' \cite{wachter2017counterfactual}.
In the context of explainable machine learning, the literature has focused on finding \emph{nearest counterfactual explanations} (i.e., instances),\footnote{A counterfactual instance can be from the dataset \cite{poyiadzi2019face, wexler2019if} or generated \cite{karimi2020model, ustun2019actionable, wachter2017counterfactual}.}~which result in the desired prediction while incurring the smallest change to the individual's feature vector, as measured by a context-dependent dissimilarity metric, $\mathrm{dist} \colon \mathcal{X} \times \mathcal{X} \to \mathbb{R}_+$.
This problem has been formulated as the following optimization problem \cite{wachter2017counterfactual}:
\begin{equation}
  \begin{aligned}
    \xCFEnearest \in \argmin_{\xAny} \quad \mathrm{dist}(\xAny, \xF)         %\\
                   \quad \mathrm{s.t.}      \quad \htheta(\xAny) \not= \htheta(\xF), %\\
                                       \xAny \in \mathcal{P},%\mathrm{lausible} \enspace,
  \end{aligned}
  \label{eqn:nearest-cf}
\end{equation}
where $\xF \in \mathcal{X}$ is the factual instance; $\xCFEnearest \in \mathcal{X}$ is a (perhaps not unique) nearest counterfactual instance; $\htheta$ is the fixed binary predictor; and $\mathcal{P}$ is an optional set of \emph{plausibility} constraints, e.g., the counterfactual instance be from a relatively high-density region of the input space~\cite{joshi2019towards, poyiadzi2019face}.

Most of the existing approaches in the counterfactual explanations literature have focused on providing solutions to the optimization problem in~\eqref{eqn:nearest-cf}, by exploring semantically meaningful distance/dissimilarity functions $\mathrm{dist}(\cdot, \cdot)$ between individuals  (e.g., $\ell_0, \ell_1, \ell_\infty$, percentile-shift), accommodating different predictive models $\htheta$ (e.g., random forest, multilayer perceptron), and realistic plausibility constraints, $\mathcal{P}$. 
In particular, \cite{dhurandhar2018explanations, mothilal2019explaining, wachter2017counterfactual} solve \eqref{eqn:nearest-cf} using gradient-based optimization; \cite{russell2019efficient,ustun2019actionable} employ mixed-integer linear program solvers to support mixed numeric/binary data; \cite{poyiadzi2019face} use graph-based shortest path algorithms; \cite{laugel2017inverse} use a heuristic search procedure by growing spheres around the factual instance; \cite{guidotti2018local, sharma2019certifai} build on genetic algorithms for model-agnostic behavior; and \cite{karimi2020model} solve \eqref{eqn:nearest-cf} using satisfiability solvers with closeness guarantees.

Although nearest counterfactual explanations provide an \emph{understanding} of the most similar set of features that result in the desired prediction, they stop short of giving explicit \emph{recommendations} on how to act to realize this set of features.
The lack of specification of the actions required to realize $\xCFEnearest$ from $\xF$ leads to uncertainty and limited agency for the individual seeking recourse.
To shift the focus from explaining a decision to providing recommendable actions to achieve recourse, \citet{ustun2019actionable} reformulated \eqref{eqn:nearest-cf} as: 
%
% \begin{equation}
%   \begin{aligned}
%     \deltaVectorStar \in \argmin_{\deltaVector} &\quad \mathrm{cost}(\deltaVector; \xF)                                %\\
%      \quad
%                                     \mathrm{s.t.} \quad \htheta(\xCFE) \not= \htheta(\xF),                     %\\
%                                                   %&\quad 
%                                                   \xCFE =  \xF + \deltaVector,                          %\\
%                                                   %&\quad 
%                                                   \xCFE \in \mathcal{P},%\mathrm{lausible},             % \\
%                                                   %&\quad 
%                                                  \deltaVector \in \mathcal{F},%\mathrm{easible} \enspace,
%   \end{aligned}
%   \label{eqn:nearest-cf-additive}
% \end{equation}
\begin{equation}
  \begin{aligned}
    \deltaVectorStar \in \argmin_{\deltaVector} \; \mathrm{cost}(\deltaVector; \xF) \quad \mathrm{s.t.} &\; \htheta(\xCFE) \not= \htheta(\xF), \\
    																									&\; \xCFE =  \xF + \deltaVector, \\
                                                                                                        &\; \xCFE \in \mathcal{P}, \; \deltaVector \in \mathcal{F},
  \end{aligned}
  \label{eqn:nearest-cf-additive}
\end{equation}
\noindent where $\mathrm{cost}(\cdot; \xF) \colon \mathcal{X} \times \mathcal{X} \to \mathbb{R}_+$ is a user-specified cost that encodes preferences between feasible actions from $\xF$, and $\mathcal{F}$ and $\mathcal{P}$ are optional sets of feasibility and plausibility constraints,\footnote{Here, ``feasible'' means \emph{possible to do}, whereas ``plausible'' means \emph{possibly true, believable or realistic}. Optimization terminology refers to both as \emph{feasibility} sets.}~restricting the actions and the resulting counterfactual explanation, respectively.
The feasibility constraints in~\eqref{eqn:nearest-cf-additive}, as introduced in~\cite{ustun2019actionable}, aim at restricting the set of features that the individual may act upon.
For instance, recommendations should not ask  individuals to change their gender or reduce their age. 
Henceforth, we refer to the optimization problem in~\eqref{eqn:nearest-cf-additive} as the \emph{CFE-based recourse} problem.

%% file: 030causal_interpretation.tex
The seemingly innocent reformulation of the counterfactual explanation problem in~\eqref{eqn:nearest-cf} as a recourse problem in~\eqref{eqn:nearest-cf-additive} is founded on two assumptions:

\vspace{1mm}

\noindent\textbf{Assumption~1:} the feature-wise difference between factual and nearest counterfactual instances, $\deltaVectorStar = \xCFEnearest - \xF$, directly translates to the minimal action set, $\actionCFE$, such that performing the actions in $\actionCFE$ starting from $\xF$ will result in $\xCFEnearest$; and 

\vspace{1mm}

\noindent\textbf{Assumption~2:} there is a 1-1 mapping between $\mathrm{dist}(\cdot, \cdot)$ and $\mathrm{cost}(\cdot; \cdot)$, whereby larger actions incur larger distance and higher cost.

\vspace{1mm}

Unfortunately, these assumptions only hold in restrictive settings, rendering the solution of~\eqref{eqn:nearest-cf-additive} \emph{sub-optimal} or \emph{infeasible} in many real-world scenarios. 
%
% Specifically, \textbf{Assumption~1} holds only
% %
% %the individual may change the value of as subset of variables (input features to the predictive models) without affecting the value of the other features, for which either (i) features are independent on each other, or (ii) the individual is able to enforce such independence via an action. 
% if the individual applies effort in a world where changing a variable does not affect other variables (i.e., no variable causally depends on the acted-upon variables); %\bernhard{to a causality reviewer, this would sound odd: the statement that intervening on one variable does not affect others is not identical ("i.e.") to a causal statement -- maybe rewrite the text in parentheses?} or if
% %
% % (ii) the individual changes the value of a subset of variables while simultaneously enforcing that the value of all other variables remain unchanged (i.e., breaking dependencies between features). 
% %
% Beyond the \emph{sub-optimality} that arises from assuming/reducing to an independent world in (i), and disregarding the \emph{feasibility} of non-altering actions in (ii), non-altering actions may naturally incur a penalty which is not captured in the current definition of cost. Hence, \textbf{Assumption~2} does not hold either.
%
Specifically, \textbf{Assumption~1} holds only if (i) the individual applies effort in a world where changing a variable does not have downstream other variables (i.e., features are independent from each other); or if (ii) the individual changes the value of a subset of variables while simultaneously enforcing that the value of all other variables remain unchanged (i.e., breaking dependencies between features). Beyond the \emph{sub-optimality} that arises from assuming/reducing to an independent world in (i), and disregarding the \emph{feasibility} of non-altering actions in (ii), non-altering actions may naturally incur a cost which is not captured in the current definition of cost, and hence \textbf{Assumption~2} does not hold either.
Therefore, except in trivial cases where the model designer actively inputs pair-wise independent features to $\htheta$, generating recommendations from counterfactual explanations in this manner, i.e., ignoring the dependencies between features, warrants reconsideration.
\hbox{Next, we formalize these shortcomings using causal reasoning.}

\subsection{Actions as  Interventions}
Let $\scmModel \in \Pi$ denote the structural causal model (SCM) capturing all inter-variable causal dependencies in the real world.
$\scmModel = \langle \zset{F}, \zset{X}, \zset{U} \rangle$ is characterized by the endogenous (observed) variables, $\zset{X} \in \mathcal{X}$, the exogenous variables, $\zset{U} \in \mathcal{U}$, and a sequence of structural equations $\zset{F} \colon \mathcal{U} \to \mathcal{X}$, describing how endogenous variables can be (deterministically) obtained from the exogenous variables~\cite{pearl2000causality,spirtes2000causation}.
Often, $\scmModel$ is illustrated using a directed graphical model, $\scmGraph$ (see, e.g., Figure~\ref{figure:working_example}).

From a causal perspective, actions may be carried out via \emph{structural interventions}, \hbox{$\actionSet \colon \Pi \to \Pi$,} which can be thought of as a transformation between SCMs \cite{pearl1994probabilistic, pearl2000causality}.
A set of interventions can be constructed as $\actionSet = \doop ( \{ \zrv{X}_i \coloneqq \actionScalar_i \}_{i \in I} )$ where $I$ contains the indices of the subset of endogenous variables to be intervened upon.
In this case, for each $i \in I$, the $\doop$-operator replaces the structural equation for the $i$-th endogenous variable  $\zrv{X}_i$ in $\zset{F}$ with $\zrv{X}_i \coloneqq \actionScalar_i$.
Correspondingly, graph surgery is performed on $\scmGraph$, severing graph edges incident on an intervened variable, $\zrv{X}_i$. %, with a single assignment corresponding to the value of the intervention, i.e., $\actionScalar_i$.
Thus, performing the actions $\actionSet$ in a world $\scmModel$ yields the post-intervention world model $\scmModel_\actionSet$ with structural equations $\zset{F}_\actionSet = \{ {F}_i \}_{~ i \not\in I} \cup \{ \zrv{X}_i \coloneqq \actionScalar_i \}_{i \in I}$. %The effect of 
Structural interventions are illustrated in Figure~\ref{figure:interventions_example}.

%From this perspective, 
% \comment{world}
\emph{Structural interventions} are used to predict the effect of actions on the world as a whole (i.e., how $\scmModel$ becomes $\scmModel_\actionSet$). In the context of recourse, we aim to model the effect of actions on one individual's situation (i.e., how $\xF$ becomes $\xSCF$) to ascertain whether or not the desirable outcome is achieved (i.e., $\htheta(\xF) \neq \htheta(\xSCF)$).
We compute individual-level effects using \emph{structural counterfactuals}~\cite{pearl2016causal}. %, as explained below.

% Assuming that $\scmModel$ factorizes as a directed acyclic graph (DAG), and with full specification of $\zset{F}$ (and $\zset{F}^{-1}$, such that $\zset{F}(\zset{F}^{-1}(\xAny)) = \xAny$),  $\zset{X}$ can be uniquely determined given the value of $\zset{U}$ (and vice-versa).

Assuming \emph{causal sufficiency} of $\scmModel$ (i.e., no hidden confounders), and full specification of an invertible $\zset{F}$ (such that $\zset{F}(\zset{F}^{-1}(\xAny)) = \xAny$),  $\zset{X}$ can be uniquely determined given the value of $\zset{U}$ (and vice-versa).
Hence, one can determine the distinct values of exogenous variables that give rise to a particular realization of the endogenous variables, $\{\zrv{X}_i = \xFScalar_i\}_i \subseteq \mathcal{X}$, as $\zset{F}^{-1}(\xF)$~\cite{pearl2016causal}.\footnote{For notational simplicity, we interchangeably use sets and vectors, e.g., $\{\zrv{X}_i = \xFScalar_i\}_i \subseteq \mathcal{X}$ and $\xF \in \mathcal{X}$.}
As a result, we can compute \emph{any} structural counterfactual query $\xSCF$ %, which automatically account for inter-variable causal dependencies,   
for an individual $\xF$ as $\xSCF = \fafinverse{\xF}$. In our context, that is: 
``if an individual $\xF$ observed in world $\scmModel$ performs the set of actions $\actionSet$,  what \emph{will be}  the resulting individual's feature vector $\xSCF$''.\footnote{Queries such as this subsume both \emph{retrospective/subjunctive/counterfactual} (``what would have been the value of'') and \emph{prospective/indicative/predictive} (``what will be the value of'') conditionals \cite{sep2014conditionals, lagnado2013causal, sep2019counterfactuals}, as long as we assume that the laws governing the world, $\zset{F}$, are stationary.}%  This is discussed in more detailmore on this in section 5}}
%``given model $\scmModel$ and having observed $\xF$,  what \emph{is} the value of all endogenous variables if the set of actions $\actionSet$ is performed".

\begin{figure}[t]
  \begin{subfigure}[c]{.55\linewidth}
    \begin{center}
      \begin{tikzpicture}
        % Nodes
        \node[state, fill=gray!60] (x1) at (0,0)               {$\zrv{X}_1$};
        \node[state, fill=gray!60] (x2) [right = 0.6 cm of x1] {$\zrv{X}_2$};
        \node[state]               (u1) [left  = 0.3 cm of x1] {$\zrv{U}_1$};
        \node[state]               (u2) [right = 0.3 cm of x2] {$\zrv{U}_2$};

        % Directed edge
        \path (u1) edge (x1);
        \path (u2) edge (x2);
        \path (x1) edge (x2);
      \end{tikzpicture}
    \end{center}
  \end{subfigure}%
  ~~\begin{subfigure}[c]{.45\linewidth}
  % \vspace{-5mm}
    {\small
      \begin{center}
        \[
          \left.
          \begin{aligned}
            \zrv{X}_1 &\coloneqq \zrv{U}_1 \\
            \zrv{X}_2 &\coloneqq f_2 (\zrv{X}_1) + \zrv{U}_2
          \end{aligned}
          \right\} ~~ \scmModel
        \]
      \end{center}
    }
  \end{subfigure} \\
  \begin{subfigure}[c]{.55\linewidth}
    \begin{center}
      \begin{tikzpicture}
        % Nodes
        \hspace{0.65cm}
        \node[inter, fill=gray!60] (x1) at (0,0)                {$\zrv{X}_1$};
        \node[state, fill=gray!60] (x2) [right = 0.5 cm of x1] {$\zrv{X}_2$};
        \node[state]               (u2) [right = 0.3 cm of x2]  {$\zrv{U}_2$};

        % Directed edge
        \path (u2) edge (x2);
        \path (x1) edge (x2);
      \end{tikzpicture}
    \end{center}
  \end{subfigure}%
  ~~\begin{subfigure}[c]{.45\linewidth}
  % \vspace{-5mm}
    {\small
      \begin{center}
        \[
          \left.
          \begin{aligned}
            \zrv{X}_1 &\coloneqq a_1 \\
            \zrv{X}_2 &\coloneqq f_2 (\zrv{X}_1) + \zrv{U}_2 
          \end{aligned}
          \right\} ~~ \scmModel_1
        \]
      \end{center}
    }
  \end{subfigure} \\
  \begin{subfigure}[c]{.55\linewidth}
    \begin{center}
      \begin{tikzpicture}
        % Nodes
        \hspace{-0.575cm}
        \node[state, fill=gray!60] (x1) at (0,0)                 {$\zrv{X}_1$};
        \node[inter, fill=gray!60] (x2) [right = 0.65 cm of x1] {$\zrv{X}_2$};
        \node[state]               (u1) [left  = 0.3 cm of x1]   {$\zrv{U}_1$};

        % Directed edge
        \path (u1) edge (x1);
      \end{tikzpicture}
    \end{center}
  \end{subfigure}%
  ~~\begin{subfigure}[c]{.45\linewidth}
  % \vspace{-5mm}
    {\small
      \begin{center}
        \[
          \left.
          \begin{aligned}
            \zrv{X}_1 &\coloneqq \zrv{U}_1 \\
            \zrv{X}_2 &\coloneqq a_2 
          \end{aligned}
          \right\} ~~ \scmModel_2
        \]
      \end{center}
    }
  \end{subfigure} \\
  \begin{subfigure}[c]{.55\linewidth}
    %\vspace{3mm}
    \begin{center}
      \begin{tikzpicture}
        % Nodes
        %\hspace{0.05cm}
        \node[inter, fill=gray!60] (x1) at (0,0)               {$\zrv{X}_1$};
        \node[inter, fill=gray!60] (x2) [right = 0.7 cm of x1] {$\zrv{X}_2$};

        % Directed edge
        % N/A
      \end{tikzpicture}
    \end{center}
  \end{subfigure}% 
  ~~\begin{subfigure}[c]{.45\linewidth}
  % \vspace{-5mm}
    {\small
      \begin{center}
        \[
          \left.
          \begin{aligned}
            \zrv{X}_1 &\coloneqq a_1 \\
            \zrv{X}_2 &\coloneqq a_2 
          \end{aligned}
          \right\} ~~ \scmModel_3
        \]
      \end{center}
    }
  \end{subfigure}
  \caption{
    Given world model, $\scmModel$, intervening on $\zrv{X}_1$ and/or on $\zrv{X}_2$ result in different post-intervention models: $\scmModel_1 = \scmModel_{\actionSet = \{\doop(\zrv{X}_1 \coloneqq a_1)\}}$ corresponds to interventions only on $\zrv{X}_1$ with consequential effects on $\zrv{X}_2$; $\scmModel_2 = \scmModel_{\actionSet = \{\doop(\zrv{X}_2 \coloneqq a_2)\}}$ shows the result of structural interventions only on $\zrv{X}_2$ which in turn dismisses ancestral effects on this variable; and, $\scmModel_3 = \scmModel_{\actionSet = \{\doop(\zrv{X}_1 \coloneqq a_1, \zrv{X}_2 \coloneqq a_2)\}}$ is the resulting (independent world) model after intervening on both variables, i.e., the type of interventions generally assumed in the CFE-based recourse problem. %\iv{reorganize}
  }
  \label{figure:interventions_example}
  \vspace{-3mm}
\end{figure}

\subsection{Limitations of CFE-based recourse}%\bernhard{should define the acronym CFE (it's implicitly clear but better be explicit)}}
\label{sec:cfe_limitations}

%We here study the actions resulting from the CFE-based recourse problem in~\eqref{eqn:nearest-cf-additive} from a causal perspective. %the a perspective of structural intervention.
%by introducing the following definition:    
Next, we use causal reasoning to formalize the limitations of the  CFE-based recourse approach in~\eqref{eqn:nearest-cf-additive}. 
To this end, we first reinterpret the actions resulting from solving the CFE-based recourse problem, i.e., $\deltaVectorStar$,  as structural interventions by defining the set of indices of observed variables that are intervened upon, $I$.  
%
%Note that different intervention sets $I$ re
%
%The set $I$ may be any arbitrary subset of observed variables, as long as it contains the variable indeces for which $\deltaScalarStar_i>0$. 
%
We remark that, given $\deltaVectorStar$, an individual seeking recourse % in a world $\scmModel$, 
may intervene on any arbitrary subset of observed variables $I$, as long as the intervention contains the variable indices for which $\deltaScalarStar_i \neq 0$.
%opt for one of two courses of action: i) intervene only on the subset of observed variables for which $\deltaScalarStar_i>0$; or ii) additionally intervene  also on variables for which $\deltaScalarStar_i=0$. 
%
%Thus, $I$ may be any arbitrary subset of observed variables, as long as it contains the variable indices for which $\deltaScalarStar_i>0$. 
%
Now, we are in a position to define CFE-based actions as interventions, i.e., 
\begin{definition}[CFE-based actions]
\label{def:CFErecourse}
Given an individual $\xF$ in world $\scmModel$, 
 the solution of~\eqref{eqn:nearest-cf-additive},  $\deltaVectorStar$, %
and the set of indices of observed variables that are acted upon, $I$,
%$\deltaVector \in \mathcal{X}$ such that $\htheta (\xF + \deltaVector) =1$
%
a \emph{CFE-based action} refers to a set of structural interventions of the form $\actionCFE := \doop ( \{ \zrv{X}_i \coloneqq x^F_i + \deltaScalarStar_i \}_{i \in I} )$. %, with $I$ denoting the set of indices of observed variables for which $\delta_i >0$. %, that results in a structural counterfactual $\xSCF = \fafinverse{\xF}$ satisfying that $\htheta (\xSCF) =1$.
\end{definition}
Using Definition \ref{def:CFErecourse}, we can derive the following key results that provide necessary and sufficient  conditions for CFE-based actions to guarantee recourse.  
\begin{restatable}[]{proposition}{CFEACTION}
\label{prop:no_children}
A CFE-based action, $\actionCFE$, where \hbox{$I = \{i ~|~ \deltaScalarStar_i \neq 0\}$}, performed by individual $\xF$, in general results in the structural counterfactual, $\xSCF= \xCFEnearest := \xF + \deltaVectorStar$, and thus guarantees recourse (i.e., \hbox{$h(\xSCF)\neq h(\xF)$}), if and only if, the set of descendants of the acted upon variables, determined by $I$, is the empty set.
\end{restatable}
%
% \begin{proof}
% proof that if it is the empty set, the equality holds, and that if  it is not the empty set the equality may not hold. q then p, p then q (or equivalently not q, then not p).
% \end{proof}
%
\begin{restatable}[]{corollary}{CFEINDEP}
\label{coro:indep_world}
If the true world $\scmModel$ is independent, i.e, all the observed features are root-nodes, 
then CFE-based actions always guarantee recourse. %Otherwise, the post-intervention SCM $\scmModel$
\end{restatable}
% \begin{proof}
% The set of descendants for each observed variable in an independent SCM is by definition the empty set. 
% \end{proof}
%
%\textit{Proof hint.}  
While the above results are formally proven in Appendix~\ref{app:proofs}, we provide a sketch of the proof below. % let us  provide an informal hint at the main points of the proofs.
If the intervened-upon variables do not have descendants, then by definition \hbox{$\xSCF = \xCFEnearest$}. Otherwise, the value of the descendants will depend on the counterfactual value of their parents, leading to a structural counterfactual that does not resemble the nearest counterfactual explanation, \hbox{$\xSCF \neq \xCFEnearest$}, and thus may not result in recourse. Moreover, in an independent world the set of descendants of all the variables is by definition the empty set. %, and changing a variable does not affect others. 

Unfortunately, the independent world assumption is not realistic, as it requires all the features selected to train the predictive model $\htheta$ to be independent of each other. Moreover, limiting changes to only those variables without descendants may unnecessarily limit the agency of the individual, e.g., in \textbf{Example~1}, restricting the individual to only changing bank balance without e.g., pursuing a new/side job to increase their income would be limiting.
Thus, for a given non-independent $\scmModel$ capturing the true causal dependencies between features, CFE-based actions require the individual seeking recourse to enforce (at least partially) an independent post-intervention model $\scmModel_{\actionCFE}$ (so that \textbf{Assumption~1} holds),  by intervening on all the observed variables for which $\delta_i \neq 0$ as well as on their descendants (even if their $\delta_i = 0$). 
However, such requirement suffers from two main issues.
First, it conflicts with \textbf{Assumption~2}, since holding the value of variables may still imply potentially \emph{infeasible} and costly interventions in $\scmModel$ to sever all the incoming edges to such variables, and even then it may not change the prediction (see \textbf{Example~2}).
Second, as will be proven in the next section (see also, \textbf{Example~1}), CFE-based actions may still be \emph{suboptimal}, as they do not benefit from the causal effect of actions towards changing the prediction.
Thus, even when equipped with knowledge of causal dependencies, recommending actions directly from counterfactual explanations in the manner of existing approaches is not satisfactory.

%% file: 040methodology.tex
In the previous section, we learned that  actions which immediately follow from counterfactual explanations may require unrealistic assumptions, or alternatively, result in sub-optimal or even infeasible recommendations. 
To solve such limitations we rewrite the recourse problem so that instead of finding the minimal (independent) shift of features as in \eqref{eqn:nearest-cf-additive}, we seek the minimal cost set of actions (in the form of structural interventions) that results in a counterfactual instance yielding the favourable output from $\htheta$:
%
% Specifically, we reformulate \eqref{eqn:nearest-cf-additive} as:
%
\begin{equation}
  \begin{aligned}
    \actionSetStar \in \argmin_{\actionSet} &\quad \mathrm{cost}(\actionSet; \xF)                    \\
                              \mathrm{s.t.} &\quad \htheta(\xSCF) \not= \htheta(\xF)                 \\
                                            &\quad \xSCF =  \fafinverse{\xF}                         \\
                                            &\quad \xSCF \in \mathcal{P},%\mathrm{lausible}            
                    \quad \actionSet \in \mathcal{F},%\mathrm{easible} \enspace,
  \end{aligned}
  \label{eqn:nearest-cf-causal}
\end{equation}
\noindent where $\actionSetStar \in \mathcal{F}$ directly specifies the set of feasible actions to be performed for minimally costly recourse, with \hbox{$\mathrm{cost}(\cdot; \xF) \colon \mathcal{F} \times \mathcal{X} \to \mathbb{R}_+$}, and $\xSCFnearest = \fastarfinverse{\xF}$  denotes the resulting structural counterfactual.
We recall that,  although $\xSCFnearest$ is a counterfactual instance, it does not need to correspond to the nearest counterfactual explanation, $\xCFEnearest$, resulting from \eqref{eqn:nearest-cf-additive} (see, e.g., \textbf{Example~1}).
Importantly, using the formulation in~\eqref{eqn:nearest-cf-causal} it is now straightforward to show the suboptimality of CFE-based actions, as shown next (proof in Appendix~\ref{app:proofs}): % in the next result. 
\begin{restatable}[]{proposition}{SUBOPTIMAL}
\label{prop:suboptimality}
Given an individual $\xF$ observed in world $\scmModel \in \Pi$,  a family of feasible actions $\mathcal{F}$, and the solution of \eqref{eqn:nearest-cf-causal},  $\actionSetStar \in \mathcal{F}$. 
Assume that there exists CFE-based action $\actionCFE \in \mathcal{F}$ that achieves recourse, i.e., $\htheta(\xF) \neq \htheta(\xCFEnearest)$.  Then, $\mathrm{cost}(\actionSetStar; \xF)  \leq \mathrm{cost}(\actionCFE; \xF) $. 
\end{restatable}
Thus, for a known causal model capturing the dependencies among observed variables, and a family of feasible interventions, the optimization problem in \eqref{eqn:nearest-cf-causal} yields \emph{Recourse through Minimal Interventions} (MINT). 
Generating minimal interventions through solving \eqref{eqn:nearest-cf-causal} requires that we be able to compute the structural counterfactual, $\xSCF$, of the individual $\xF$ in world $\scmModel$, given \emph{any} feasible action, $\actionSet$.
To this end, we consider that the SCM $\scmModel$ falls in the class of additive noise models (ANM), so that we can deterministically compute the counterfactual $\xSCF = \fafinverse{\xF}$ by performing the  \emph{Abduction-Action-Prediction} steps proposed by~\citet{pearl2016causal}.
%
% Appendix~\ref{app:working_example} provides a detailed example on how to perform these three steps, as summarized in the following closed-form assignment equation:  

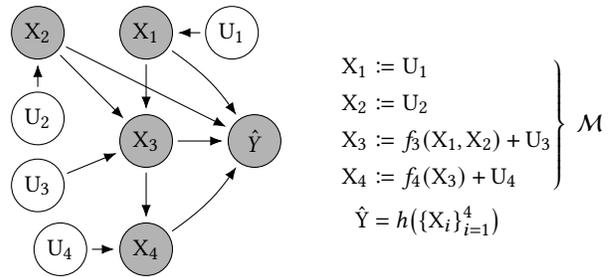
\begin{figure}[t]
  \begin{subfigure}[c]{.45\linewidth}
    \begin{center}
      \begin{tikzpicture}
        % x node set with absolute coordinates
        \node[state, fill=gray!60] (x3) at (0,0) {$\zrv{X}_3$};
        \node[state, fill=gray!60] (x1) [above = 0.60cm of x3] {$\zrv{X}_1$};
        \node[state, fill=gray!60] (x2) [left  = 0.60cm of x1] {$\zrv{X}_2$};
        \node[state, fill=gray!60] (x4) [below = 0.60cm of x3] {$\zrv{X}_4$};
        \node[state, fill=gray!60] (y)  [right = 0.60cm of x3] {$\hat{Y}$};

        \node[state] (u1) [right = 0.3cm of x1] {$\zrv{U}_1$};
        \node[state] (u2) [below = 0.3cm of x2] {$\zrv{U}_2$};
        \node[state] (u3) [left =  0.1cm of x3, below =  0.05cm of u2] {$\zrv{U}_3$};
        \node[state] (u4) [left =  0.3cm of x4] {$\zrv{U}_4$};

        % Directed edge
        \path (x1) edge (x3);
        \path (x2) edge (x3);
        \path (x3) edge (x4);
        \path (x1) edge [bend left = 10] (y);
        \path (x2) edge (y);
        \path (x3) edge (y);
        \path (x4) edge [bend right = 10] (y);

        \path (u1) edge (x1);
        \path (u2) edge (x2);
        \path (u3) edge (x3);
        \path (u4) edge (x4);
      \end{tikzpicture}
    \end{center}
  \end{subfigure}%
  \begin{subfigure}[c]{.57\linewidth}
    \begin{center}
      \[
      \left.
      \begin{aligned}
        \zrv{X}_1 &\coloneqq \zrv{U}_1                              \\
        \zrv{X}_2 &\coloneqq \zrv{U}_2                              \\
        \zrv{X}_3 &\coloneqq f_3 (\zrv{X}_1, \zrv{X}_2) + \zrv{U}_3 \\
        \zrv{X}_4 &\coloneqq f_4 (\zrv{X}_3) + \zrv{U}_4          ~~\\
      \end{aligned}
      \right\} ~ \scmModel
      \]
      \[
      \begin{aligned}
        \hspace{-8mm}\hat{\zrv{Y}} &= \htheta \big( \{ \zrv{X}_i \}_{i=1}^4 \big)
      \end{aligned}
      \]
    \end{center}
  \end{subfigure}
%   \caption{Working example; see  Section \ref{section:illustration}.} % and Appendix~\ref{app:working_example} for details.}
    \caption{The structural causal model (graph and equations) for the working example and demonstration in Section~\ref{sec:040methodology}.}
  \vspace{-3mm}
  \label{figure:working_example_extended}
\end{figure}

\subsection{Working example}
Consider the model in Figure \ref{figure:working_example_extended},
% and assume that the SCM falls in the  class of additive noise models (ANM), 
where $\{\zrv{U}_i\}_{i=1}^4$ are mutually independent exogenous variables, and $\{f_i\}_{i=1}^4$ are structural (linear or nonlinear) equations.
Let $\xF = [\xFScalar_1, \xFScalar_2, \xFScalar_3, \xFScalar_4]^T$ be the observed features belonging to the (factual) individual, for whom we seek a counterfactual explanation and recommendation.
Also, let $I$ denote the set of indices corresponding to the subset of endogenous variables that are intervened  upon according to the action set $\actionSet$.
Then, we obtain a structural counterfactual, $\xSCF = \fafinverse{\xF}$, by applying the Abduction-Action-Prediction steps~\cite{pearl2013structural} as follows:
% of counterfactual reasoning~\cite{pearl2013structural} as:

\textbf{Step 1. Abduction} uniquely determines the value of all exogenous variables, $\{u_i\}_{i=1}^4$, given evidence, $\{\zrv{X}_i = \xFScalar_i\}_{i=1}^4$:
\begin{equation}
  \begin{aligned}
    u_1 &= \xFScalar_1 ,                                \\
    u_2 &= \xFScalar_2 ,                                \\
    u_3 &= \xFScalar_3 - f_3(\xFScalar_1, \xFScalar_2), \\
    u_4 &= \xFScalar_4 - f_4(\xFScalar_3) .
  \end{aligned}
  \label{equation:working_example_abduction}
\end{equation}

\textbf{Step 2. Action} modifies the SCM according to the hypothetical interventions, $\doop(\{\zrv{X}_i \coloneqq \actionScalar_i\}_{i \in I})$ (where $a_i = x^F_i + \deltaScalar_i$), yielding $\zset{F}_{\actionSet}$: % (where $a_i$ may be $x^F_i + \deltaScalar_i$):
\begin{equation}
  \begin{aligned}
    \zrv{X}_1 &\coloneqq [1 \in I] \cdot a_1 + [1 \notin I] \cdot \zrv{U}_1 , \\
    \zrv{X}_2 &\coloneqq [2 \in I] \cdot a_2 + [2 \notin I] \cdot \zrv{U}_2 , \\
    \zrv{X}_3 &\coloneqq [3 \in I] \cdot a_3 + [3 \notin I] \cdot \big( f_3(\zrv{X}_1, \zrv{X}_2) + \zrv{U}_3 \big), \\
    \zrv{X}_4 &\coloneqq [4 \in I] \cdot a_4 + [4 \notin I] \cdot \big( f_4(\zrv{X}_3) + \zrv{U}_4 \big) ,
  \end{aligned}
  \label{equation:working_example_action}
\end{equation}
where $[\cdot]$ denotes the Iverson bracket.

\textbf{Step 3. Prediction} recursively determines the values of all endogenous variables based on the computed exogenous variables $\{u_i\}_{i=1}^4$ from Step 1 and $\zset{F}_{\actionSet}$ from Step 2, as:
\begin{equation}
  \begin{aligned}
    \xSCFScalar_1 &\coloneqq [1 \in I] \cdot a_1 + [1 \notin I] \cdot \big( u_1 \big) , \\
    \xSCFScalar_2 &\coloneqq [2 \in I] \cdot a_2 + [2 \notin I] \cdot \big( u_2 \big) , \\
    \xSCFScalar_3 &\coloneqq [3 \in I] \cdot a_3 + [3 \notin I] \cdot \big( f_3(\xSCFScalar_1, \xSCFScalar_2) + u_3 \big), \\
    \xSCFScalar_4 &\coloneqq [4 \in I] \cdot a_4 + [4 \notin I] \cdot \big( f_4(\xSCFScalar_3) + u_4 \big) .
  \end{aligned}
  \label{equation:working_example_prediction}
\end{equation}

\subsection{General assignment formulation}
As we have not made any restricting assumptions about the structural equations (only that we operate with additive noise models\footnotemark~where noise variables are pairwise independent), the solution for the working example naturally generalizes to SCMs corresponding to other DAGs with more variables.
The assignment of structural counterfactual values can generally be written as:
% Putting it all together, we obtain the following closed form expression:
% in \eqref{eqn:general_assignment_formulation_hard}.

\footnotetext{We remark that the presented formulation also holds for more general SCMs (for example where the exogenous variable contribution is not additive) as long as the sequence of structural equations $\zset{F}$ is invertible, i.e., there exists a sequence of equations $\zset{F}^{-1}$ such that $\xAny = \ffinverse{\xAny}$ (in other words, the exogenous variables are uniquely identifiable via the abduction step).}

\begin{equation}
  \begin{aligned}    
    \xSCFScalar_i = ~& [i \in I] \cdot (\xFScalar_i + \deltaScalar_i) \\
                     &~~ + [i \notin I] \cdot \big( \xFScalar_i + f_i(\pascf_i) - f_i(\paf_i) \big). 
  \end{aligned}
  \label{eqn:general_assignment_formulation_hard}
\end{equation}
% \[
% \xSCFScalar_i =
% \begin{cases}
% 	\xFScalar_i + \deltaScalar_i              & i \in I    \\
% 	\xFScalar_i + f_i(\pascf_i) - f_i(\paf_i) & i \notin I \\
% \end{cases}
% \]
%
In words,  the counterfactual value of the $i$-th feature, $\xSCFScalar_i$, takes the value $\xFScalar_i +\deltaScalar_i$ if such feature is intervened upon (i.e., $i \in I$). Otherwise, $\xSCFScalar_i$ is computed as a function of both the factual and counterfactual values of its parents, denoted respectively by $f_i(\paf_i)$ and $f_i(\pascf_i)$. 
The closed-form expression in \eqref{eqn:general_assignment_formulation_hard} can replace the counterfactual constraint in \eqref{eqn:nearest-cf-causal}, i.e., $\xSCF = \fafinverse{\xF}$, after which the optimization problem may be solved by building on existing frameworks for generating nearest counterfactual explanations, including gradient-based, evolutionary-based, heuristics-based, or verification-based approaches as referenced in Section \ref{sec:020background}.
While out of scope of the current work, for the demonstrative examples below, we extended the open-source code of MACE \cite{karimi2020model}; we will submit a pull-request to the respective repository.

% \hl{Supporting both forms of interventions is important ..., we review work that attempts to incorporate causal relations in eq 2 in  related work...}

% \textbf{Remark \#1:} We assume the time horizon is such that the effects of all interventions have played out through the entire model, before the result is again evaluated on $\htheta$. E.g., the bank asks the individual to wait a year before re-applying for a loan, to allow for income-based savings to accumulate.

% \textbf{Remark \#2:} Importantly, 
% Finally, the closed-form relations in \eqref{eqn:general_assignment_formulation_hard} and \eqref{eqn:general_assignment_formulation_soft} can replace the counterfactual constraint in \eqref{eqn:nearest-cf-causal}, i.e., $\xSCF = \fafinverse{\xF}$, after which the optimization problem may be solved by building on existing frameworks for generating nearest counterfactual explanations, including gradient-based, evolutionary-based, heuristics-based, or verification-based approaches as referenced in Section \ref{sec:020background}.
% %
% While out of scope of the current work, for the demonstrative examples below, we extended the open-source code of MACE \cite{karimi2020model}; we will submit a pull-request to the respective repository.

\input{045demonstration}

%% file: 045demonstration.tex
\subsection{Demonstration}
\label{section:illustration}
We showcase our proposed formulation by comparing the actions recommended by existing (nearest) counterfactual explanation methods, as in \eqref{eqn:nearest-cf-additive}, to the ones generated by the proposed minimal intervention formulation in~\eqref{eqn:nearest-cf-causal}. 
We recall that prior literature has focused on generating counterfactual explanations or CFE-based actions, which as shown above lack optimally or feasibility guarantees in non-independent worlds.
Thus, to the best of our knowledge, there exists no baseline approach in the literature that guarantees algorithmic recourse.
The experiments below serve as an illustration of the sub-optimality of existing approaches relative to our proposed formulation of recourse via minimal intervention. Section~\ref{sec:060discussion} presents a detailed discussion on practical considerations.

We consider two settings: i) a synthetic setting where $\scmModel$ follows Figure~\ref{figure:working_example}; and ii) a real-world setting based on the \texttt{german} credit dataset~\cite{bache2013uci}, where $\scmModel$ follows Figure~\ref{figure:working_example_extended}.
We computed the cost of actions as the $\ell_1$ norm over normalized feature changes to make effort comparable across features, i.e., $\mathrm{cost}(\cdot; \xF) = \sum_{i \in I} |\deltaScalar_i| / R_i$, where $R_i$ is the range of feature $i$.

For the \emph{synthetic setting}, we generate data following the model in Figure~\ref{figure:working_example}, where we assume $\zrv{X}_1 \coloneqq \zrv{U}_1$, $\zrv{X}_2 \coloneqq 3 / 10 \cdot \zrv{X}_1 + \zrv{U}_2$, with $\zrv{U}_1 \sim \$10000 \cdot \mathrm{Poission}(10)$ and  $\zrv{U}_2 \sim \$2500 \cdot \mathcal{N}(0, 1)$; 
and the predictive model $\htheta = \mathrm{sgn}(\zrv{X}_1 + 5 \cdot \zrv{X}_2 - \$225000)$. %, as described in \S\ref{sec:010introduction}.
Given $\xF = [\$75000, \$25000]^T$, solving our formulation, \eqref{eqn:nearest-cf-causal}, identifies the optimal action set $\actionSetStar = \doop(\zrv{X}_1 \coloneqq \xFScalar_1 + \$10000)$ which results in $\xSCFnearest = \fastarfinverse{\xF} = [\$85000, \$28000]^T$, whereas solving previous formulations, \eqref{eqn:nearest-cf-additive}, yields $\deltaVectorStar = [\$0, +\$5000]^T$ resulting in $\xCFEnearest = \xF + \deltaVectorStar = [\$75000, \$30000]^T$.
Importantly, while $\xSCFnearest$ appears to be at a further distance from $\xF$ compared to $\xCFEnearest$, achieving the former is less costly than the latter, specifically,  $\mathrm{cost}(\deltaVectorStar; \xF) \approx 2 ~ \mathrm{cost}(\actionSetStar; \xF)$.

As a \emph{real-world setting}, we consider a subset of the features in the \texttt{german} credit dataset. 
The setup is depicted in Figure \ref{figure:working_example_extended}, where $\zrv{X}_1$ is the individual's gender (treated as immutable), $\zrv{X}_2$ is the individual's age (actionable but can only increase), $\zrv{X}_3$ is credit given by the bank (actionable), $\zrv{X}_4$ is the repayment duration of the credit (non-actionable but mutable), and $\hat{\zrv{Y}}$ is the predicted customer risk, according to $\htheta$ (logisitic regression or decision tree).
We learn the structural equations by fitting a linear regression model to the child-parent tuples.
We will release the data, and the code used to learn models and structural equations.

Given the setup above, for instance, for the individual $\xF = [\texttt{Male}, 32, \$1938, 24]^T$ identified as a risky customer, solving our formulation, \eqref{eqn:nearest-cf-causal}, yields the optimal action set $\actionSetStar = \doop(\{\zrv{X}_2 \coloneqq \xFScalar_2 + 1, \zrv{X}_3 \coloneqq \xFScalar_3 - \$800\})$ which results in $\xSCFnearest = \fastarfinverse{\xF} = [\texttt{Male}, 33, \$1138, 22]^T$, whereas solving \eqref{eqn:nearest-cf-additive} yields $\deltaVectorStar = [\texttt{N/A}, +6, 0, 0]^T$ resulting in $\xCFEnearest = \xF + \deltaVectorStar = [\texttt{Male}, 38, \$1938, 24]^T$. 
%
% \comment{0.015469 vs. 0.02678}
Similar to the toy setting, we observe a $\%42$ decrease in effort required of the individual when using the action by our method, since our cost function states that waiting for six years to get the credit approved is more costly than applying the following year for a lower ($-\$800$) credit amount.  
%
% More generally,
We extend our analysis to a population level, and observe that for 50 negatively affected test individuals, previous approaches suggest actions that are on average $\%39 \pm \%24$ and $\%65 \pm \%8$ more costly than our approach when considering, respectively, a logistic regression and a decision tree as the predictive model $\htheta$. 

The demonstrations above confirm our theoretical analysis that MINT-based actions from \eqref{eqn:nearest-cf-causal} are less costly and thus more beneficial for affected individuals than existing CFE-based actions from \eqref{eqn:nearest-cf-additive} that fail to utilize the causal relations between variables.

%% file: 050realistic.tex
In Section \ref{sec:040methodology}, we formulated algorithmic recourse by considering the causal relations between features in the real world.
Our formulation minimized the cost of actions, which were carried out as \emph{structural} interventions on the corresponding graph. %, $\scmGraph$.
Each intervention proceeds by \emph{unconditionally} \emph{severing all edges} incident on the intervened node, fixing the post-manipulation distribution of a \emph{single} variable to \emph{one deterministic} value.
While intuitive appealing and powerful, structural interventions are in many ways the simplest type of interventions, and their ``simplicity comes at a price: foregoing the possibility of modeling many situations realisitically''~\cite{eberhardt2007causation, korb2004varieties}.
Below, we extend \eqref{eqn:nearest-cf-causal} and \eqref{eqn:general_assignment_formulation_hard} to add flexibility and realism to the types of interventions performed by the individual.
Notably, there is nothing inherent to an SCM that a priori determines the \emph{form}, \emph{feasibility}, or \emph{scope} of intervention; instead, these choices are delegated to the individual and are made based on a semantic understanding of the modeled variables.

\subsection{On the Form of Interventions}

The demonstrations in Section~\ref{section:illustration} primarily focused on actions performed as \emph{structural (a.k.a., hard)} interventions~\cite{pearl2000causality} where all incoming edges to the intervened node are severed (see \eqref{eqn:general_assignment_formulation_hard}).
Hard interventions are particularly useful for Randomized Control Trial (RCT) settings where one aims to evaluate (isolate) the causal effect of an action (e.g., effect of aspirin on patients with migraine) on the population by randomly assigning individuals to treatment/control groups, removing the influence of other factors (e.g., age).

In the context of algorithmic recourse, however, an individual performs actions in the real world, and therefore must play the rules governing the world.
In earlier sections, these rules (captured in an SCM) guided the search for an optimal set of actions by modelling actions along with their consequences.
The rules also determine the form of an intervention, e.g., specifying whether an intervention cancels out or complements existing causal relations. % existing causal effect relations.

For instance, consider \textbf{Example~1}, where an individual chooses to increase their bank balance (e.g., through borrowing money from family, i.e., a deliberate action/intervention on $\zrv{X}_2$ while continuing to put aside a portion of their income (i.e., retaining the relation $\zrv{X}_2 \coloneqq 3 / 10 \cdot \zrv{X}_1 + \zrv{U}_2$).
Indeed, it would be unwise for a recommendation to suggest abandoning saving habits.
In such a scenario, the action would be carried out as an \emph{additive (a.k.a., soft)} intervention~\cite{eberhardt2007interventions}.
Such interventions \emph{do not} sever graphical edges incident on the intervened node and continue to allow for parents of the node to affect that node.
%
% Complementarily
Conversely, in \textbf{Example~2}, recourse recommendations may suggest performing a structural intervention on temperature, e.g., by creating a climate controlled green-house, to cancel the natural effect of altitude change on temperature.

The previous examples illustrate a scenario where an individual/agriculture team actually have the agency to choose which type of intervention to perform. % a structural intervention, but prefers an additive intervention instead.
However, it is easy to conceive of examples where such an option does not exist.
For instance, as part of a medical system's recommendation, we might consider adding $5$ \si{mg/l} of insulin to a patient with diabetes with a certain blood insulin level~\cite{pearl2016causal}.
This action cannot disable pre-existing mechanisms regulating blood insulin levels and therefore, the action can only be performed additively.
Conversely, one may also consider another example from the medical domain whereby the only treatment of malignancy may be through a surgical (structural) amputation.\footnote{See, e.g., \url{https://www.cancer.org/cancer/bone-cancer/treating/surgery.html}.}

Just as structural interventions were supported in our framework via a closed-form expression (see \eqref{eqn:general_assignment_formulation_hard}), additive interventions can be encoded through an analogous assignment formulation:
% Additive interventions are easily handled in our framework, where the general assignment formulation \eqref{eqn:general_assignment_formulation_hard} is updated for variables that can be intervened upon in an additive manner, as in \eqref{eqn:general_assignment_formulation_soft}.
% \vspace{2mm}
\begin{equation}
  \begin{aligned} % \nonumber
    % \xSCFScalar_i = [i \in I] \cdot (x^F_i + \deltaScalar_i) + \big( \xFScalar_i + f_i(\pascf_i) - f_i(\paf_i) \big) \\
    \xSCFScalar_i = [i \in I] \cdot \deltaScalar_i + \big( \xFScalar_i + f_i(\pascf_i) - f_i(\paf_i) \big) \\
  \end{aligned}.
  \label{eqn:general_assignment_formulation_soft}
\end{equation}

The choice of whether interventions should be applied in a additive/soft or structural/hard manner depends on the variable semantic~\cite{barocas2020hidden}, and should be decided prior to solving \eqref{eqn:nearest-cf-causal}.

\subsection{On the Feasibility of Interventions}

We saw in Section \ref{sec:030causal_interpretation} that earlier works motivated the addition of \emph{feasibility} constraints as a means to provide more actionable recommendations for the individual seeking recourse~\cite{ustun2019actionable}.
There, the \emph{actionability} (a.k.a. \emph{mutability}) of a feature was determined based on the feature semantic and value in the factual instance, marking those features which the individual has/lacks the agency to change (e.g., bank balance vs. race).
While the interchangeable use of definition holds under an independent world, it fails when operating in most real-world settings governed by a set of causal dependencies.
We study this subtlety below.

In an independent world, any change to variable $\zrv{X}_i$ could come about only via an intervention on $\zrv{X}_i$ itself. Therefore, immutable and non-actionable variables overlap.
In a dependent world, however, changes to variable $\zrv{X}_i$ may arise from an intervention on $\zrv{X}_i$ or through changes to any of the ancestors of $\zrv{X}_i$. %, i.e., $\an_i$.
In this more general setting, we can tease apart the definition of \emph{actionability} and \emph{mutability}, and distinguish between three types of variables:
(i) immutable (and hence non-actionable), e.g., race;
(ii) mutable but non-actionable, e.g., credit score; and
(iii) actionable (and hence mutable), e.g., bank balance.
Each type requires special consideration which we show can be intuitively encoded as constraints amended to $\actionSet \in \mathcal{F}$ from \eqref{eqn:nearest-cf-causal}.

\textbf{Immutable:}
We posit that the set of immutable (and hence non-actionable) variables should be closed under ancestral relationships given by the model, $\scmModel$.
This condition parallels the ancestral closure of \emph{protected} attributions in~\cite{kusner2017counterfactual}.
This would ensure that under no circumstance would an intervention on an ancestor of an immutable variable change the immutable variable.
Therefore, for an immutable variable $\zrv{X}_i$, the constraint $[i \notin I] = 1 $ recursively necessitates the fulfillment of additional constraints $[j \notin I] = 1 ~ \forall ~ j \in \pa_i$ in $\mathcal{F}$.
For instance, the immutability of race triggers the immutability of birthplace.

\textbf{Mutable but non-actionable:}
To encode the conditions for mutable but non-actionable variables, we note that while a variable may not be directly actionable, it may still change as a result of changes to its parents.
For example, the financial credit score in Figure \ref{figure:working_example_extended} may change as a result of interventions to salary or savings, but is not itself directly intervenable.
Therefore, for a non-actionable but mutable variable $\zrv{X}_i$, the constraint $[i \notin I] = 1$ is sufficient and does not induce any other constraints.

\textbf{Actionable:}
In the most general sense, the actionable feasibility of an intervention on $\zrv{X}_i$ may be contingent on a number of conditions, as follows:
(a) the pre-intervention value of the intervened variable (i.e., $\xFScalar_i$);
(b) the pre-intervention value of other variables (i.e., $\{ \xFScalar_j \}_{j \subset [d] \setminus i}$);
(c) the post-intervention value of the intervened variable (i.e., $\xSCFScalar_i$); and
(d) the post-intervention value of other variables (i.e., $\{ \xSCFScalar_j \}_{j \subset [d] \setminus i}$).
Such feasibility conditions can easily be encoded into $\mathcal{F}$; consider the following scenarios:

(a) an individual's age can only increase, i.e.,
$[\xSCFScalar_{age} \ge \xFScalar_{age}]$;
(b) an individual cannot apply for credit on a temporary visa, i.e.,\\
\hbox{$[\xFScalar_{visa} = \texttt{PERMANENT}] \ge [\xSCFScalar_{credit} = \texttt{TRUE}]$};

(c) an individual may undergo heart surgery (an additive intervention) only if they won't remiss due to sustained smoking habits, i.e.,
$[\xSCFScalar_{heart} \not= \texttt{REMISSION}]$; and

(d) an individual may undergo heart surgery only \emph{after} their blood pressure is regularized due to medicinal intervention, i.e.,
$[\xSCFScalar_{bp} = \texttt{O.K.}] \ge [\xSCFScalar_{heart} = \texttt{SURGERY}]$.

In summary, while previous works on algorithmic recourse distinguished between actionable, conditionally actionable,\footnotemark~and immutable variables~\cite{ustun2019actionable}, we can now operate on a more realistic \emph{spectrum} of variables, ranging from conditionally soft/hard actionable, to non-actionable but mutable, and finally to immutable and non-actionable variables.
Finally, we remind that feasibility is a distinct notion from plausibility; whereas the former restricts actions $\actionSet \in \mathcal{F}$ to those that can be performed by the individual, the latter determines the likeliness of the counterfactual instance $\xSCF = \fafinverse{\xF} \in \mathcal{P}$ resulting from those actions.
For instance, building on the earlier example, although an individual with similar attributes and higher credit score may exist in the dataset (i.e., plausible), directly acting on credit score is not feasible.

\footnotetext{\citet{ustun2019actionable} also support conditionally actionable features (e.g., age or educational degree) with conditions derived only from $\xFScalar_i$ as in (a). We generalize the set of conditions to support actions conditioned on the value of other variables as in (b), additive interventions in (c), and sequential interventions as in (d).}

\subsection{On the Scope of Interventions}

One final assumption has been made throughout our discussion of actions as interventions which pertain to the one-to-one mapping between an action in the real world and an intervention on a endogenous variable in the structural causal model (which in turn are also input features to the predictive model).
As exemplified in~\cite{barocas2020hidden}, it is possible for some actions (e.g., finding a higher-paying job) to simultaneously intervene on multiple variables in the model (e.g., income and length of employment).
Alternatively, for \textbf{Example~2}, choosing a new paddy location is equivalent to intervening jointly on several input features of the predictive model (e.g., altitude, radiation, precipitation).
Such confounded/correlated interventions, referred to as \emph{fat-hand}/\emph{non-atomic} interventions~\cite{eberhardt2007interventions}, will be explored further in follow-up work, by modelling the world at different causally consistent levels~\cite{beckers2019abstracting, rubenstein2017causal}.

%% file: 060discussion.tex
In this paper, we have focused on the problem of algorithmic recourse, i.e., the process by which an individual can change their situation to obtain a desired outcome from a machine learning model.
First, using the tools from causal reasoning (i.e., structural interventions and counterfactuals), we have shown that in their current form, counterfactual explanations only bring about agency for the individual to achieve recourse in  unrealistic  settings.
In other words, counterfactual explanations do not translate to an \emph{optimal} or \emph{feasible} set of actions that would favourably change the prediction of $\htheta$ if acted upon.
This shortcoming is primarily due to 
the lack of consideration of causal relations governing the world and thus, the failure to model the downstream effect of actions in the predictions of the machine learning model. 
In other words, although ``counterfactual'' is a term from causal language, we observed that existing approaches fall short in terms of taking causal reasoning into account when generating counterfactual explanations and the subsequent recourse actions.
Thus, building on the statement by~\citet{wachter2017counterfactual} that counterfactual explanations ``do not rely on knowledge of the causal structure of the world,'' it is perhaps more appropriate to refer to existing approaches as \emph{contrastive}, rather than \emph{counterfactual}, explanations~\cite{dhurandhar2018explanations, miller2019explanation}.

To directly take causal consequences of actions into account, we have proposed a fundamental reformulation of the recourse problem, 
where actions are performed as interventions
and we seek to minimize the cost of performing actions in a world governed by a set of (physical) laws captured in a structural causal model.
Our proposed formulation in \eqref{eqn:nearest-cf-causal}, complemented with several examples and a detailed discussion, allows for \emph{recourse through minimal interventions} (MINT), that when performed will result in a \emph{structural counterfactual} that favourably changes the output of the model.

Next, we discuss the work most closely related to ours, the main limitation of the proposed recourse approach, and propose future venues for research to address such shortcomings. 

\textbf{Related work. }
A number of authors have argued for the need to consider causal relations between variables \cite{wachter2017counterfactual, ustun2019actionable, karimi2020model, mothilal2019explaining}, generally based on the intuition that changing some variables may have effects on others. % \hl{independently may not be feasible}. + 2018 paper
In the original counterfactual explanations work, \citet{wachter2017counterfactual} also suggest that ``counterfactuals generated from an accurate causal model may ultimately be of use to experts (e.g., to medical professionals trying to decide which intervention will move a patient out of an at-risk group)''.
Despite this general agreement, to the best of our knowledge, only two works have attempted to technically formulate this requirement. % using. the need of causal constraints in the formulation of nearest counterfactuals.

In the first work, \citet{joshi2019towards} study recourse in causal models under confounders and with predetermined treatment variables.
In this work, a distribution over hidden confounders is first estimated along with a mapping from the attributes $\xAny$ to hidden confounders, i.e., $G_\theta^{-1}(\xAny) = \zvec{z}$.
Then, under each intervention on treatment variables, explanations are generated following \eqref{eqn:nearest-cf} with the plausibility term constraining the inverse of the counterfactual instance (i.e., $G_\theta^{-1}(\xAny)$) to the approximated confounding distribution. % recourse is solved similar to (1)
% Here, counterfactual distributions are first estimated under a fixed intervention on predetermined treatment variables. Then, under each intervention, recourse is solved similar to \eqref{eqn:nearest-cf} with the plausibility term constraining the counterfactual instance to the approximated counterfactual distribution. 
%
% Similar to the first related work, this work also suffers from the same limitations as other CFE-based recourse approaches presented in Section~\ref{sec:cfe_limitations} in that a returned counterfactual instance would not imply feasible or optimal actions for recourse.
%
In this work, we instead optimize for recourse actions rather than counterfactual instances that result from those action.

% \amir{still argumenting...}
In the second work, \citet{mahajan2019preserving} present a modified version of the distance function in \eqref{eqn:nearest-cf}, amending the \emph{standard proximity loss} between factual and counterfactual instances with a \emph{causal regularizer} to encourage the counterfactual value of each endogenous variable to be close to the value of that variable had it been assigned via its structural equation. %, i.e.,
%\begin{equation*}
%   \begin{aligned}
%     \mathrm{distCausal}(\xF, \xCFE) = \sum_{i \in \zset{U}} \text{dist}(\xF_i, \xCFE_i) + \sum_{i \in \zset{X}} \text{dist}(\xCFEScalar_i, f_i(\pa_i^\texttt{CFE})) ,
%   \end{aligned}
%   \label{equation:mahajan}
% \end{equation*}
% that counterfactual values of variables be similar to the value they would take under additive interventions to their ancestors. %
%
%\noindent where $\zset{U}$ and $\zset{X}$ correspond to exogenous and endogenous variables, respectively, and $f_i(\pa^\texttt{CFE}) = \mathbb{E}[\xCFEScalar_i ~ | ~ \pa_i^\texttt{CFE}]$ is the approximated structural equation. 
%
Beyond the uncertainty regarding the strength of regularization (which would mean causal relations may not be guaranteed), and why the standard proximity loss only iterates over the exogenous variables (which from a causal perspective, are characteristics that are shared across counterfactual worlds~\citep[footnote 4]{kusner2017counterfactual}), this approach suffers from a primary limitation in its causal treatment:
the causal regularizer would penalize any variable whose value deviated away from its structurally assigned value.
While on the surface this ``preservation of causal relations'' seems beneficial, such an approach would discourage interventions (additive or structural) on non-root variables, which would, by design, change the value of the intervened-upon variable away from its structurally assigned value.
Instead, the regularizer would encourage interventions on variables that would not be penalized as such, i.e., root variables, which may not be contextually acceptable as root notes typically capture sensitive characteristic of the individual (e.g., birthplace, age, gender).
The authors suggest (in the Appendix of \cite{mahajan2019preserving}) that one may consider those variables, upon which (structural) interventions are to be performed, as exogenous. In this manner, interventions would not be penalized and down-stream effects of interventions would still be preserved when searching for the nearest counterfactual instance.
We argue, however, that such an approach suffers from the same limitations as other CFE-based recourse approaches presented in Section~\ref{sec:cfe_limitations} in that a returned counterfactual instance would not imply feasible or optimal actions for recourse.
Finally, without an explicit abduction step and without assumptions on the form of structural equations, it is unclear how the authors infer and combine individual-specific characteristics (as embedded in the background variables) with the effect of ancestral changes to compute the counterfactual. % an intervention assignment
We believe the problems above will be mostly resolved when minimizing over the cost of actions instead of distance over counterfactuals as we have done in this work.

\textbf{Practical limitations.}
The primary limitation of our formulation in~\eqref{eqn:nearest-cf-causal} is its reliance on the true causal model of the world, subsuming both the graph, and the structural equations.
In practice, the underlying causal model is rarely known, which suggests that the counterfactual constraint in \eqref{eqn:nearest-cf-causal}, i.e., $\xSCF = \fafinverse{\xF}$, may not be (deterministically) identifiable.
We believe this is a valid criticism, not just of our work, but of any approach suggesting actions to be performed in the real world for consequential decision-making.
Importantly, beyond recourse, the community on algorithmic fairness has echoed the need for causal counterfactual analysis for fair predictions, and have also voiced their concern about untestable assumptions when the true SCM is not available \cite{barabas2017interventions, chiappa2019path, kilbertus2017avoiding, kusner2017counterfactual, russell2017worlds}.

% more we argue that existing approaches already implicitly make causal assumptions (i.e., that of independence, or feasible and cost-free interventions).
%
% Therefore, at worst case we replace an imperfect assumption about the data generative process with another.

Perhaps more concerningly, our work highlights the implicit causal assumptions made by existing approaches (i.e., that of independence, or feasible and cost-free interventions), which may portray a false sense of recourse guarantees where one does not exists (see \textbf{Example~2} and all of Section \ref{sec:cfe_limitations}).
%
% What our works aims to contribute 
Our work aims to highlight existing imperfect assumptions, and to offer an alternative formulation, backed with proofs and demonstrations, which would guarantee recourse if assumptions about the causal structure of the world were satisfied.
Future research on causal algorithmic recourse may benefit from the rich literature in causality that has developed methods to verify and perform inference under various assumptions \cite{peters2017elements}.
% localize the assumptions to the causal generative process.} which 
%
Thus, we consider further discussion on causal identifiability to be out of scope of this paper, as it remains as an open and key question in the Ethical ML community.

This is not to say that counterfactual explanations should be abandoned altogether. On the contrary, we believe the counterfactual explanations hold promise for ``guided audit of the data'' \cite{wachter2017counterfactual} and evaluating various desirable model properties, such as robustness \cite{sharma2019certifai, hancox2020robustness} or fairness \cite{sharma2019certifai, gupta2019equalizing, ustun2019actionable, karimi2020model}.
Besides this, it has been shown that designers of interpretable machine learning systems use counterfactual explanations for predicting model behavior \cite{lage2019evaluation} or uncovering inaccuracies in the data profile of individuals \cite{venkatasubramanianphilosophical}.
Complementing these offerings of counterfactual explanations, we offer minimal interventions as a way to guarantee algorithmic recourse in general settings, which is not implied by counterfactual explanations.
% our work merely aims to caution against the general use of counterfactual explanations as actionable recourse recommendations, and for that we offer MINT.

% Moreover, we emphasize that inferring interventional effects in causal models depends on modelling assumptions and which variables are deemed endogenous vs exogenous.
% %
% Our formulation optimizes for interventions on endogenous variables that are directly actionable by the individual seeking recourse or an assigned fiduciary~\cite{venkatasubramanianphilosophical}.
% %
% This assumption, can be extended to consider changes to exogenous variables via actions of a collective of individuals in a social context~\cite{miller2019explanation}. An example of this may be the active protests in the U.S.
% %
% We consider that the social context and collective actions leads to the data that our algorithms are trained on for which we seek algorithmic recourse.
% %

\textbf{Future work. }
In future work, we aim to focus on overcoming the main assumption of our formulation: the availability of the true world model, $\scmModel$.
An immediate first step involves learning the true world model (partially or fully)~\cite{eberhardt2017introduction, glymour2019review, malinsky2018causal}, and studying potential inefficiencies that may arise from partial or imperfect knowledge of the causal model governing the world.
Furthermore, while additive noise models are a broadly used class of SCMs for modeling real-world systems, further investigation into the effects of confounders (non-independent noise variables), the presence of only the causal graph, as well as cyclic graphical models for time series data (e.g., conditional interventions), would extend the reach of algorithmic recourse to even broader settings.

In Section~\ref{sec:050realistic}, we presented feasibility constraints for a wide range of settings, including dynamical settings in which one intervention enables the preconditions of another.
An interesting line of future research would involve combining the causal intervention-based recourse framework, as presented in our work, with multi-stage planning strategies such as \cite{ramakrishnan2019synthesizing} to generate optimal sequential actions.
%
% \hl{other types of actions? beyond those from SCM? planning?}
% related?
% http://adres.ens.fr/IMG/pdf/08032002.pdf
% https://arxiv.org/pdf/1905.01195.pdf

Finally, the examples presented in relation to the form and feasibility of intervention serve only to illustrate the flexibility of our formulation in supporting a variety of real-world constraints. % conditions/constraints.
They do not, however, aim to provide an authoritative definition of how to interpret variables and the context- and individual-dependent constraints for recourse as highlighted by other works~\cite{barocas2020hidden, kohler2018eddie}.
Future cross-disciplinary research would benefit from accurately defining the variables and relationships and types of permissible interventions in consequential decision-making settings.
Relatedly, future research would also benefit from a study of properties that cost functions should satisfy (e.g., individual-based or population-based, monotonicity) as the primary means to measure the effort endured by the individual seeking recourse.

%% file: appendix/proofs.tex
\newcommand{\ndesc}{\texttt{nd}}
\newcommand{\desc}{\texttt{d}}

\subsection{Proof of Proposition \ref{prop:no_children}}

\CFEACTION*

\begin{proof}
The setting assumes that the causal graph $\scmGraph$ is available such that the parent set for each variable is known.
Let $\desc(X)$ and $\ndesc(X)$ denote the sets of descendants and non-descendants of the variable $X$ according to $\scmGraph$, respectively.
For multiple intervened-upon variables, we define: $$\zset{X}_I:=\{X_i\}_{i \in I},$$ $$\ndesc(\zset{X}_I):=\cap_{i \in I} \ndesc(X_i),$$ $$\desc(\zset{X}_I):=\zset{X}\setminus(\zset{X}_I\cup \ndesc(\zset{X}_I)).$$ 
Note that, by definition, $\zset{X}_I$, $\ndesc(\zset{X}_I)$, and $\desc(\zset{X}_I)$ form a partition of the set of all variables $\zset{X}$.

To prove the iff conditional, we prove each direction separately. For ease of exposition, we define 
\[
\underbrace{\xSCF= \xCFEnearest := \xF + \deltaVectorStar}_{\mathbf{p}} \quad \Longleftrightarrow \quad \underbrace{\desc(\zset{X}_I) = \varnothing}_{\mathbf{q}}
\]

\noindent where we recall the remark that given $\deltaVectorStar$, an individual seeking recourse may intervene on any arbitrary subset of observed variables $\zset{X}_I$, as long as $(\deltaScalarStar_i \neq 0) \implies (i \in I)$.

$\mathbf{q \Longrightarrow p}$: Borrowing the closed-form expression of a structural counterfactual from \eqref{eqn:general_assignment_formulation}, we have

\begin{equation}
\xSCFScalar_i = \begin{cases}
    \xFScalar_i + \deltaScalarStar_i          & i \in I     \\
    \xFScalar_i + f_i(\pascf_i) - f_i(\paf_i) & i \not\in I
\end{cases}
\end{equation}
\noindent which can be broken down further to specify the descendants and non-descendants of intervened upon variables, as 
\begin{equation}
\xSCFScalar_i = \begin{cases}
    \xFScalar_i + \deltaScalarStar_i          & i \in I                        \\
    % \xFScalar_i + f_i(\pascf_i) - f_i(\paf_i) & i \text{ descendant of } I     \\
    % \xFScalar_i + f_i(\pascf_i) - f_i(\paf_i) & i \text{ non-descendant of } I \\
    \xFScalar_i + f_i(\pascf_i) - f_i(\paf_i) & i \in \desc(\zset{X}_I)  \\
    \xFScalar_i + f_i(\pascf_i) - f_i(\paf_i) & i \in \ndesc(\zset{X}_I) \\
\end{cases}
\label{eqn:app_assignment}
\end{equation}

By assumption, $\desc(\zset{X}_I) = \varnothing$, so the second case never holds. 

Furthermore, since structural interventions leave non-descendant variables unaffected, we have that $$\pascf_i = \paf_i \quad \forall i\in \ndesc(\zset{X}_I).$$
Consequently, $$f_i(\pascf_i) - f_i(\paf_i) = f_i(\paf_i) - f_i(\paf_i) = 0 \quad \forall i\in \ndesc(\zset{X}_I).$$

In summary, we have
\begin{equation}
\xSCFScalar_i = \begin{cases}
    \xFScalar_i + \deltaScalarStar_i            & i \in I                        \\
    \xFScalar_i                            & i \in \ndesc(\zset{X}_I) \\
\end{cases}
\end{equation}

\noindent which, upon realising that $(\deltaScalarStar_i \neq 0) \implies (i \in I)$, reduces to $\xSCF = \xCFEnearest := \xF + \deltaVectorStar$ as desired.

$\mathbf{\neg q \Longrightarrow \neg p}$: Starting with the negation of $\mathbf{q}$, we have the $\exists ~ k \in I \text{ s.t. } \desc(X_k) \not= \varnothing$. It is assumed that $\deltaScalarStar_k \not=0$ (i.e., we are not performing a non-altering intervention on $X_k$), then using the same expression for structural counterfactuals in \eqref{eqn:app_assignment}, there in general exists a descendant of $X_k$ for which the value of its ancestors change under intervention, i.e., $\exists ~ l \in \desc(\zset{X}_I) \text{ s.t. } f_l(\pascf_l) - f_l(\paf_l) \not= 0$. Thus, $\xSCFScalar_l \not= \xFScalar_l$ and thus $\xSCF \not= \xCFEnearest := \xF + \deltaVectorStar$.
Our proof ignores special cases such as piece-wise constant structural equations, where for some $\deltaScalarStar_i \neq 0$, the descendant of $X_i$ remains invariant. These rare cases can be thought of as locally violating causal minimality \cite[Sec. 6.5]{peters2017elements} and are thus disregarded.
\end{proof}

\subsection{Proof of Corollary \ref{coro:indep_world}}

\CFEINDEP*

\begin{proof}
If the true world $\scmModel$ is independent, then by definition the set of descendants for all variables is the empty set. Thus, the statement follows directly from Proposition \ref{prop:no_children}.
\end{proof}

\subsection{Proof of Proposition \ref{prop:suboptimality}}

\SUBOPTIMAL*

\begin{proof}
Having assumed that both $\actionCFE, \actionSetStar \in \mathcal{F}$, and considering that $\actionSetStar$ is the optimal solution of \eqref{eqn:nearest-cf-causal} constrained to $\mathcal{F}$, it follows from definition of optimality that $\mathrm{cost}(\actionSetStar; \xF)  \leq \mathrm{cost}(\actionCFE; \xF) $.
\end{proof}